\documentclass[pdflatex,sn-basic,iicol]{sn-jnl}

\jyear{2021}%

\theoremstyle{thmstyleone}%
\newtheorem{theorem}{Theorem}
\newtheorem{corollary}{Corollary}[theorem]

\theoremstyle{thmstyletwo}%

\theoremstyle{thmstylethree}%
\newtheorem{definition}{Definition}%
\DeclareMathOperator*{\argmax}{arg\,max}
\DeclareMathOperator*{\argmin}{arg\,min}

\usepackage{natbib}
\usepackage{amsfonts}
\usepackage{amssymb}
\usepackage{amsmath}
\usepackage{graphicx}
\usepackage{subcaption}

\begin{document}

\title[DDPNAS]{DDPNAS: Efficient Neural Architecture Search via Dynamic Distribution Pruning}


\author[1,2]{\fnm{Xiawu} \sur{Zheng}}\email{zhengxiawu@stu.xmu.edu.cn}

\author[1]{\fnm{Chenyi} \sur{Yang}}\email{31520220157234@stu.xmu.edu.cn}

\author[1]{\fnm{Shaokun} \sur{Zhang}}\email{shaokunzhang529@gmail.com}

\author[5]{\fnm{Yan} \sur{Wang}}\email{yan.wang@samsara.com}

\author[6]{\fnm{Baochang} \sur{Zhang}}\email{bczhang@buaa.edu.cn}

\author[7]{\fnm{Yongjian} \sur{Wu}}\email{littlekenwu@tencent.com}

\author[7]{\fnm{Yunsheng} \sur{Wu}}\email{simonwu@tencent.com}

\author[8]{\fnm{Ling} \sur{Shao}}\email{ling.shao@ieee.org}

\author*[1,2,3,4]{\fnm{Rongrong} \sur{Ji}}\email{rrji@xmu.edu.cn}



\affil*[1]{Media Analytics and Computing Laboratory, Department of Artificial Intelligence, School of Informatics, Xiamen University, China}

\affil[2]{Peng Cheng Laboratory}

\affil[3]{Institute of Artificial Intelligence, Xiamen University, 361005, P.R. China}

\affil[4]{Fujian Engineering Research Center of Trusted Artificial Intelligence Analysis and Application, Xiamen University, 361005, P.R. China}

\affil[5]{Samsara, Seattle, WA}

\affil[6]{Beihang University}

\affil[7]{BestImage Lab, Tencent Co., Ltd., Shanghai 200233, China}

\affil[8]{Terminus Group, China}


\abstract{Neural Architecture Search (NAS) has {demonstrated} state-of-the-art performance on \textcolor{black}{various} computer vision tasks. 
{Despite \textcolor{black}{the} superior performance achieved, the efficiency and generality of existing methods are highly valued due to \textcolor{black}{their} high computational complexity and low generality.} 
In this paper, we propose an efficient and unified NAS framework termed DDPNAS via dynamic distribution pruning, {facilitating} a theoretical bound on accuracy and efficiency. 
In particular, we first sample architectures from a joint categorical distribution. {Then the search space is dynamically pruned and its distribution is updated every few epochs.} 
{With \textcolor{black}{the} proposed efficient network generation method, we directly obtain the optimal neural architectures on given constraints, which is practical for on-device models across diverse search space\textcolor{black}{s} and constraints.} 
The architectures searched by our method achieve remarkable top-1 accuracies, $97.56$ and $77.2$ on CIFAR-10 and ImageNet (mobile settings),  respectively, with the fastest search process, \emph{i.e.}, only $1.8$ GPU hours on a Tesla V100. Codes for searching and network generation are available at: \url{https://openi.pcl.ac.cn/PCL_AutoML/XNAS}}.

\keywords{Neural architecture search, Dynamic distribution pruning, Efficient network generation.}



\maketitle

\section{Introduction}\label{sec: intro}

\begin{figure*}[htb]
\center
\includegraphics[width=1.0\linewidth]{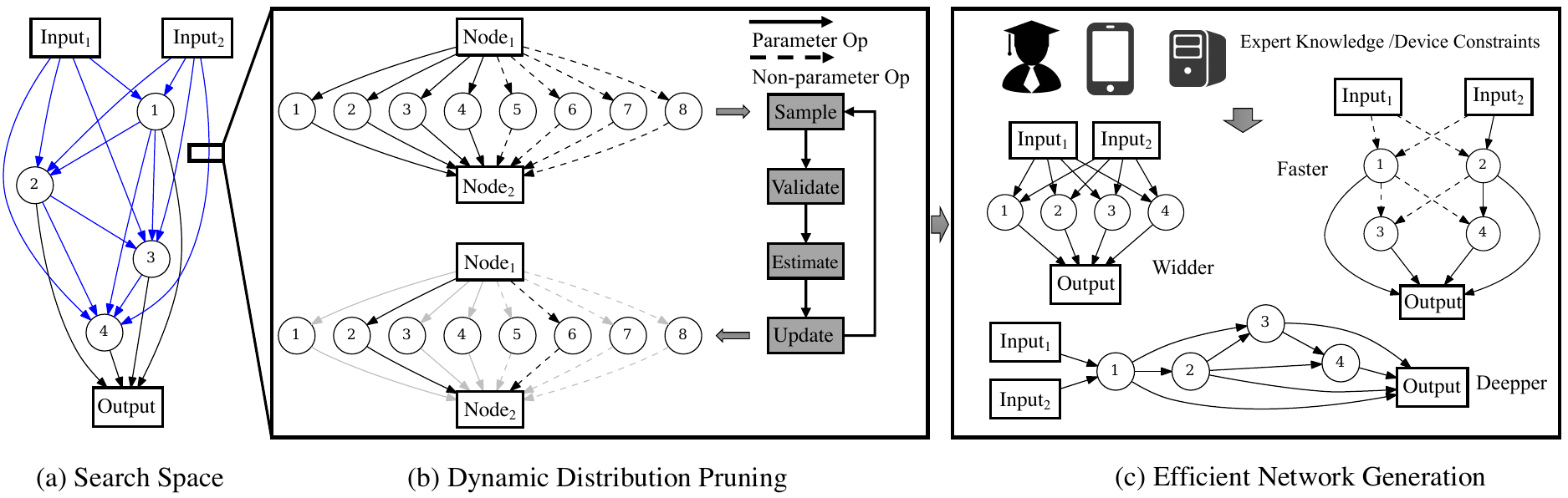}
\caption{\label{fig:alg} The overall framework of the proposed Dynamic Distribution Pruning and Efficient Network Generation. (a) The search cell is a fully connected directed acyclic graph with four nodes. (b) The proposed Dynamic Distribution Pruning method first samples architectures from the search space according to the corresponding distribution. Then, the generated network is trained over $T$ epochs with forward and backward propagation. The distribution is estimated by testing the network on the validation set. Finally, the element with the lowest probability is pruned from the candidates. The pruning step is ceased when only one parameter layer and one non-parameter layer retained in the search space. (c) Given a specific constraint, the best architecture is the one with maximum expectation in the pruned search space.}
\end{figure*}

Deep Neural Networks (DNNs) discovered by Neural Architecture Search (NAS) have demonstrated superior performance \textcolor{black}{than} handcrafted architectures {on various computer vision tasks, such as image classification \citep{zoph2016neural}, object detection \citep{ghiasi2019fpn} and segmentation \citep{liu2019auto}. NAS automates architecture engineering by searching state-of-the-art neural network architectures without human intervention over a vast architecture search space.}

Despite the remarkable effectiveness, existing NAS methods are still limited \textcolor{black}{by} the high demand on computational and memory costs during the architecture search process. Reinforcement Learning (RL) based methods \citep{zoph2018learning}\citep{zoph2016neural} {often take thousands of GPU days to search optimal architectures from certain search space.} To accelerate the training {process}, one-shot methods {focused on weight-sharing based methods. 
DARTS and its variants continuously parameterize the \textcolor{black}{architecture} distribution, enabling the gradient-descent based optimization to reduce search time, while retaining a comparable accuracy. Recent approaches further improve the efficiency of architecture search by progressively increasing the layers \citep{chen2019progressive} or reducing channels \citep{xu2019pc}.} 
However, the scope of these methods is also limited, \emph{e.g.}, the number of layers and channels are fixed in a chain-structure search space \citep{cai2018proxylessnas,cai2019once,wu2019fbnet,howard2019searching}.

{More importantly, \textcolor{black}{regarding} applying NAS on various resource-constrained platforms, especially on wearable electronics, IoT devices and mobile devices, research concerning on \textcolor{black}{efficient} searching of architectures under diverse constraints (\emph{e.g.,} model size, latency and network depth, etc.) becomes more urgent. However, this task is \textcolor{black}{significantly} time-consuming using conventional approaches, since the computational budgets surge linearly with an increase in the amount of constraints.} 
{To handle this issue, recent works~\citep{cai2019once,yu2019network} propose a ``once for all" paradigm, which trains an over-parameterized graph. And then, it is further inferred on different devices to find the optimal architectures. However, even maintaining the inference step is computationally expensive.} 
{For example, $16, 000$ samples with different input image sizes and architectures are required in previous work \citep{cai2019once}. The accuracy of sampled architectures \textcolor{black}{is} evaluated on $10, 000$ validation images from the original training set.}

{To alleviate the aforementioned challenges,} we present a unified and efficient NAS algorithm, which introduces a novel Dynamic Distribution Pruning (DDP) with Efficient Network Generation (ENG) to achieve extremely efficient and generalized NAS. 
In principle, we consider architectures as being sampled from a dynamic joint categorical distribution. 
Subsequently, a dynamic distribution $p(\theta)$ is introduced to control the \textcolor{black}{operations choices}, and a specific network architecture is directly sampled.
In the search process, we generate different samples and train them on the training set over a few epochs. 
Then, the evaluation results on the validation set are used to estimate categorical distribution. 
The element with the lowest probability is dynamically pruned. 
After the pruning process, only one parameter layer and one non-parameter layer {are} left in the search space. 
Given a specific constraint, the best architecture is the one \textcolor{black}{that} contains the highest expectation in the constrained search space. 
{Both theoretical and quantitative investigations are performed to validate the efficacy of our proposed method.}
Fig.~\ref{fig:alg} shows the overall framework. Our contributions are summarized as follows:
\begin{itemize}
\item {We introduce a novel dynamic distribution pruning NAS strategy (DDPNAS), which is flexible, memory-efficient, and generalized in different search spaces. Instead of designing algorithms on a specific search space like other methods \citep{zoph2018learning,liu2018darts},} we directly encode the search space to a probability distribution, which enables NAS to deal with different search spaces with a comparable computation cost to the training of conventional DNNs. Experiments on cell-wise, layer-wise and BERT-compression search space also validate the effectiveness of DDPNAS. In addition, our model can be also easily incorporated into most existing NAS algorithms to accelerate the search process. 
\item {A novel network generation method is proposed, termed as Efficient Network Generation, which is effective and generalizable to both device constraints and manual constraints. } 
\item We provide a theoretical interpretation and bound for the proposed method, which quantitatively demonstrates the rationale of the dynamic pruning design.
\item Extensive experiments show that DDPNAS achieves remarkable search efficiency, \emph{e.g.}, $2.44$\% test error on CIFAR10 \citep{krizhevsky2010convolutional} after only $1.8$ hours of search\textcolor{black}{ing} with one Tesla V100. When evaluated on ImageNet \citep{russakovsky2015imagenet}, DDPNAS can directly search over the full dataset within two days, achieving $77.2$ \% top-1 accuracy under the MobileNet settings.
\end{itemize}

\begin{table*}[t]
\small
\begin{center}
\setlength{\tabcolsep}{1mm}{
\begin{tabular}{c|ccccc}
\hline
\multirow{2}{*}{Algorthims} & Search   & Search    & Latency & Search          & {Complexity} \\ \cline{2-6} 
                           & Method   & Effectiveness & Aware   & Space           & {on} N {devices}          \\ \hline
NASNet \citep{zoph2018learning}                     & RL       &          &        & cell/layer-wise & $\times$ N          \\
AmoebaNet   \citep{real2018regularized}               & EA       &         &        & cell/layer-wise & $\times$ N         \\
ENAS \citep{pham2018efficient}                       & RL       & \checkmark        &        & cell            &  $\times$ N          \\
DARTS  \citep{liu2018darts}                    & gradient & \checkmark         &        & cell            & $\times$ N          \\
PC-darts \citep{xu2019pc}                  & gradient & \checkmark         &        & cell            & $\times$ N          \\
ProxylessNAS  \citep{cai2018proxylessnas}             & gradient & \checkmark         & \checkmark       & cell/layer-wise & $\times$ N          \\
Once-for-all  \citep{cai2019once}             & training &          & \checkmark       & layer-wise      & $+$ N          \\
Our DDPNAS                 & pruning  & \checkmark         & \checkmark       & cell/layer-wise/{BERT-compression} & $\times 1$          \\ \hline
\end{tabular}}
\end{center}
\caption{Comparison of the proposed DDPNAS with other state-of-the-art NAS algorithms. ``Search method" indicates different types of algorithms in the search phase. ``Search Effectiveness": We consider an algorithm effective if it takes $\leq 4$ GPU hours on CIFAR-10 and $\leq 4$ GPU days on ImageNet. ``Latency aware":  The algorithm is considered to be latency aware when it considers latency in its optimization. ``Search space" denotes the search scope on cell-based and layer-wise search space. ``Complexity on N devices": The computation cost with $N$ different devices/constraints. }
\label{tab:merit}
\end{table*}

\section{Related Work}

As an automatic machine learning technique, neural architecture search has recently attracted significant attention in computer vision. 
For a given dataset, architectures with high accuracy or low latency are obtained by performing a heuristic search in a predefined search space. For simple tasks like image classification, most human-designed networks are built by stacking \emph{reduction} (\emph{i.e.,} the spatial dimension is reduced and the channel size is increased) and \emph{norm} (\emph{i.e.,} the spatial and channel dimensions are preserved) cells \citep{he2016deep,krizhevsky2012imagenet,simonyan2014very,huang2017densely,hu2018squeeze, chen2021binarized}.
Therefore, existing NAS methods \citep{zoph2016neural,zoph2018learning,liu2018progressive,liu2018darts} can search architectures under the same settings to work on a small search space.

Various search algorithms have been proposed to explore the neural architecture space using specific search strategies. 
One popular approach is to model NAS as a \emph{Reinforcement Learning} (RL) problem \citep{zoph2016neural,zoph2018learning,baker2016designing,cai2018efficient,liu2018progressive,cai2018path}. 
Zoph \emph{et al.} \citep{zoph2018learning} employ a recurrent neural network as the policy function to sequentially generate a string that encodes the specific neural architecture. 
The policy network can be trained with the policy gradient algorithm or the proximal policy optimization.
\citep{cai2018efficient,cai2018path} proposed a method that regards the search space as a tree structure for network transformation. 
In this method, new network architectures can be generated by a father network with certain predefined transformations, which reduces the search space and speeds up the search. 
An alternative way to explore the architecture space is through \emph{evolutionary} methods, which generate a population of network architectures using evolutionary algorithms \citep{xie2017genetic, real2018regularized}. 
Although the above architecture search algorithms have achieved state-of-the-art results on various tasks, a large amount of computational resources are needed.

To overcome this problem, several recent works have been proposed to accelerate NAS in a \emph{one-shot} setting, which can find the optimal network architecture within a few GPU days. 
In this one-shot architecture search, each architecture in the search space is considered as a sub-graph sampled from a super-graph, and the search process can be accelerated by parameter sharing \citep{pham2018efficient}. 
\citep{liu2018darts} jointly optimized the weights within two nodes with the hyper-parameters under continuous relaxation. 
Both the weights in the graph and the hyper-parameters are updated via standard gradient descent.
However, the method in \citep{liu2018darts} still suffers from large GPU memory footprints, and the search complexity is still not applicable to many resource-limited scenarios. 
To this end, \citep{cai2018proxylessnas} adopted the differentiable framework and proposed to search architectures without any proxy. 
However, this method is still in the same line as \citep{liu2018darts}.

Compared to the existing schemes mentioned above, the proposed DDPNAS merits in several aspects, as itemized in Tab.~\ref{tab:merit}. We first formulate NAS in an entirely new way, where the operation selection is considered as a sample from a dynamic categorical distribution. 
And the optimal distribution can be obtained through a progressively pruning. After that, we can obtain the optimal architecture under arbitrary constraints by the proposed efficient network generation, which is extremely efficient, generalizing to different search spaces and constraints/devices.

\section{The Proposed Method}

\subsection{Architecture Search Space} \label{sec:3.1}
\textbf{Cell-based search space.} {We follow the same architecture search space design} as in \citep{liu2018darts,zoph2016neural,zoph2018learning}. A network consists of a pre-defined number of cells \citep{zoph2016neural}, which can be either norm cells or reduction cells. 
Each cell takes the outputs of the two previous cells as input. A cell is a fully-connected Directed Acyclic Graph (DAG) of $M$ nodes, \emph{i.e.}, $\{B_1, B_2, ..., B_M\}$. Each node $B_i$ takes the dependent nodes as input, and generates an output through a sum operation $B_j = \sum_{i<j} o^{(i,j)}(B_i).$ 
Here each node is a specific tensor (\emph{e.g.,} a feature map in convolutional neural networks) and each directed edge $(i,j)$ between $B_i$ and $B_j$ denotes an operation $o^{(i,j)} (.)$, which is sampled from the corresponding search space $\mathcal{O}$. Note that the constraint $i<j$ ensures there will be no cycles in a cell. 
We set the two input nodes as $B_{-1}$ and $B_{0}$ for simplicity. 
Following \citep{liu2018darts}, the operation search space $\mathcal{O}$ consists of $K = 8$ operations, which we divide into two sections. Specifically, the parameter operation $\mathcal{O}_p$ includes a $3\times3$ dilated convolution with rate $2$, $5\times5$ dilated convolution with rate $2$, $3\times3$ depth-wise separable convolution, and $5\times5$ depth-wise separable convolution. The non-parameter operation $\mathcal{O}_n$ includes $3\times3$ max pooling, $3\times3$ average pooling, no connection (zero), and a skip connection (identity). 
Therefore, the size of the whole search space is $2 \times K^{\mathcal{E_M}}$, where $\mathcal{E_M}$ is the set of possible edges with $M$ intermediate nodes in the fully-connected DAG. In our case with $M=4$, together with the two input nodes, the total number of cell-based \textcolor{black}{architectures} in the search space \textcolor{black}{is} $2 \times 8^{2+3+4+5} = 2 \times 8^{14}$, which is an extremely large space to search, and requires an efficient optimization strategy.

\textbf{Layer-wise search space.} {MobileNet based networks is utilized \citep{howard2017mobilenets,sandler2018mobilenetv2,howard2019searching} as the backbone. The search space is built with different widths and \textcolor{black}{heights} as in \citep{wu2019fbnet,cai2018proxylessnas,cai2019once}. Specifically, in \citep{cai2018proxylessnas}, {each depth-wise convolution layer in MobileNet is allowed to vary kernel sizes and expansion rates within $\{3, 5, 7\}$ and $\{3, 6\}$, respectively.} Besides, by accessing the zero operation from the set of candidate operators, the over-parameterized network allows the skipping of each ``norm" layer (stride $1$), which enables a \textcolor{black}{direct} trade-off between depth and width. Thus, {the architectures could be adjusted from} 1) ``fatter" or ``thinner" by varying expansion ratios; 2) ``taller" or ``shorter" in choosing different number of zero operations. Regarding our case in $\vert \mathcal{E_M}\vert = 20$ edges/layers and $K=2 \times 3 + 1$ operators/choices, our search space therefore is $K^{\vert \mathcal{E_M}\vert} = 7^{20}$, an efficient optimization \textcolor{black}{strategy} is also required to search in this extremely large space.}

\begin{figure}[t]
\center
\includegraphics[width=1.0\linewidth]{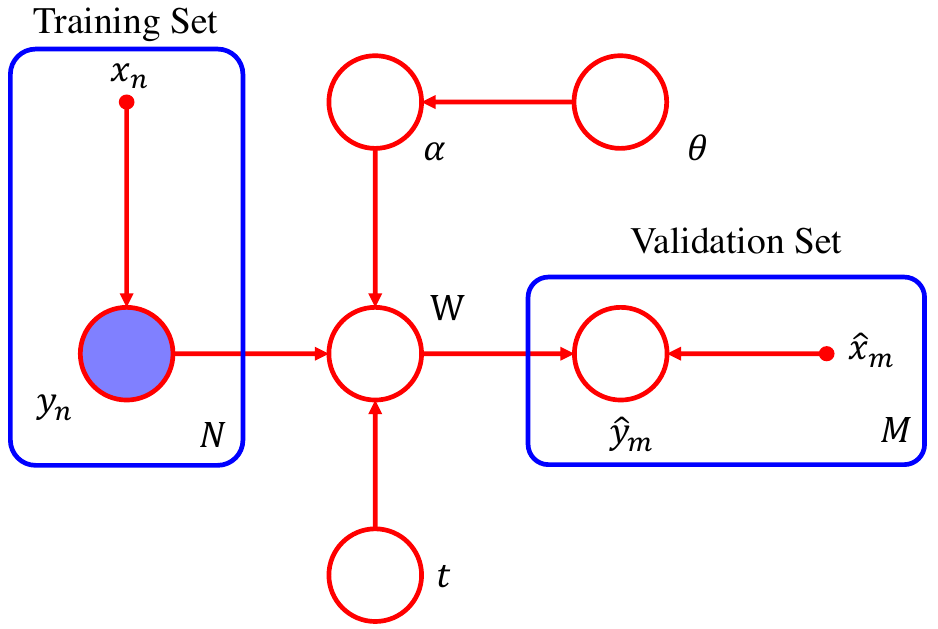}
\caption{\label{fig:prob}
{Directed graphical model representing the proposed joint distribution. Following the previous work \citep{bishop2006pattern}, the random variables in this model is denoted by the hollow circles. Plates (boxes \textcolor{black}{labeled} $N$ and $M$) represent $N$ or $M$ nodes of which only a single example $(x_n, y_n)$ is shown explicitly. We also assign the deterministic parameters as smaller \textcolor{black}{solid} nodes, while we shade nodes to indicate the corresponding random variables have been set to their observed (training set) values.}
}
\end{figure}

\subsection{Dynamic Distribution Pruning} \label{sec:3.2}
As illustrated in Fig.~\ref{fig:prob}, we first use a directed graph to describe the proposed probability distribution. 
More specifically, $\forall \quad(i,j) \in \mathcal{E_M}$, we introduce a categorical distribution $p\left(\alpha^{(i,j)} \vert \theta^{(i,j)}\right)$ defined on $\mathcal{A}$, where $\theta^{(i,j)}$ is a vector of length $K$. 
Thus, the sampling process begins from the latent state $\theta$. 
Then, each operation $o^{(i,j)}(.)$ is selected as the $k$-th operation with a probability $\theta^{(i,j)}_k$. 
We use an over-parameterized parent network \citep{liu2018darts,pham2018efficient} containing all possible operations at each edge with a weighted probability $\alpha^{(i,j)} \in \mathbb{R}^K$: $B_j = \sum_{i<j} \sum_{ o(.) \in \mathcal{O} }\alpha^{(i,j)} o^{(i,j)}(B_i),$ where $ j>0,i \geq -1 ,\{ i,j,k\} \in \mathbb{N}$. This design allows the neural architecture search to be optimized through stochastic gradient descent (SGD) by iteratively fixing $\alpha$ to update the network parameters $W$, and fixing $W$ to update $\alpha$.
While the real-value weights $\alpha$ bring convenience in optimization \citep{liu2018darts}, they {also require to evaluate every} possible operation in $\mathcal{O}$, which causes an impractically long training time. Instead, we set $\alpha^{(i,j)}$ to be a one-hot indicator vector:
\begin{equation}\label{eq:prob}
\alpha^{(i,j)} =
\left\{
\begin{array}{lr}
\left[1,0, ...,0 \right] \, {\rm with} \, {\rm probability} \, \theta^{(i,j)}_1,\\
\quad \quad ...\\
\left[0,0, ...,1 \right] \, {\rm with} \, {\rm probability} \, \theta^{(i,j)}_K,
\end{array}
\right.
\end{equation}
which is sampled from a categorical distribution, $p(\alpha \vert \theta) = \prod_{i<j}\text{Cat}(\alpha^{(i,j)}\vert\theta^{(i,j)})$. 
Although it brings significant speed-up, this discrete weight design also leads to considerable difficulty in optimization. 

As one of our core contributions, we propose to optimize the distribution using the validation likelihood as a target $p(\hat{y})$. Based on the training data $(X_\text{train} = \{x_1, ..., x_N\}, Y_\text{train} = \{y_1, ... y_N\})$, we are able to train the model according to the epoch random variable $t$ to obtain the parameter set $W$, which finally allows us to test on the validation set $X_\text{val} = \{\hat{x}_1, ..., \hat{x}_N\}$ and obtain the validation \textcolor{black}{likelihood} (also \textcolor{black}{known} as predictive distributions in \citep{bishop2006pattern}) $p(\hat{y})$. Our ultimate goal is to find an optimal distribution $\theta$ that has \textcolor{black}{the} largest \textcolor{black}{likelihood} on validation set, which involves standard sampling, training, evaluating and updating processes, as illustrated in Fig.~\ref{fig:alg}. 

While we reduce the computational requirement by $\lvert \mathcal{O} \rvert$ times when using discrete $\alpha$s, such a procedure is still time-consuming considering the large search space and the complexity of network training. Inspired by \citep{ying2019bench}, we further propose to use a dynamic pruning process to boost efficiency by a large margin. 
{\cite{ying2019bench} reach an observation that the validate \textcolor{black}{performance} ranking of different architectures in the early stage of training could not be a reliable indicator of the final quality of candidate architectures.}
Based on this intuition, we derive a simple yet effective pruning process: During training, along with the increasing epoch index $t$, we progressively prune the worst performing model. As analyzed later on in Sec.~\ref{sec:3.4}, this strategy produces good theoretical bounds.

{Suppose there are $E$ edges and $K$ candidates in the search space, $\alpha \in \{0,1\}^{E\times K}$, in which each row is a one-hot encoded \textcolor{black}{that} represents the operation choice of the corresponding edge. 
We also use an index $e$ to represent a specific edge $(i,j)$ for simplicity.
In this \textcolor{black}{case}, each row of $\theta \in \mathbb{R}^{E\times K}$ represents the probability distribution of the corresponding edge. 
Our ultimate goal is to find an optimal probability distribution that can fully represent the search space and the corresponding performance. With the optimized distribution, architectures that have a higher performance are more likely to be sampled, and vice versa. We propose to use simple yet effective steps, sampling, validating, estimating, and updating to find the optimal probability distribution. The detailed descriptions are listed as follow.
    \begin{itemize}
        \item \textbf{Sampling.} According to the probability graph model constructed in Fig.\ref{fig:prob}, we consider that the training epoch $t$ will largely affect the final performance. This, in turn, affects the final estimation of specific network architecture. We thus propose to disjointly sample $K$ architectures to alleviate this problem. Specifically, in each epoch $t$, we perform sampling without replacement according to $\theta$. In this way, each operation is only trained for one epoch. Intuitively, each operation shares \textcolor{black}{exactly} the same training epoch, which eliminates the impact of $t$ on performance estimation.
        \item \textbf{Validating.} We then train these $K$ architectures for one epoch and estimate the performance on the validation set.
        \item \textbf{Estimating.}
        When architecture \textcolor{black}{achieves} higher performance, we expect to increase the probability of the operation selected by this architecture, \textcolor{black}{and} vice versa. Although the operation performs well in this architecture, the impact of other operations should be considered. We propose to alternately repeat sampling and validating steps $T$ times to improve the estimation of each operation. 
        Specifically, we first introduce a performance matrix $P \in \mathbb{R}^{E\times K \times T}$ to record the sampled architectures and the corresponding performance. In a certain training epoch $t$, we convert $\alpha \in \mathbb{R}^{E\times K}$ to a vector $\alpha_{\text{vec}} \in \mathbb{R}^{E\times 1}$, where each element \textcolor{black}{represents} the index of the selected operations. For a specific edge $e$, we record the performance as 
        \begin{equation}\label{eq:info_reocrd}
            P[e, \alpha_{\text{vec}}[e], t] = p(\hat{y} \vert \alpha, t),
        \end{equation}
        where $p(\hat{y}\vert \alpha, t)$ is the validation performance of the architecture $\alpha$ on epoch $t$. And the utility matrix $\overline{P} \in \mathbb{R}^{E\times K}$ for all operations and edges is then obtained through
        \begin{equation}\label{eq:utility_matrix}
            \overline{P} =\frac{1}{T} P\boldsymbol{\vec{1}},
        \end{equation}
        where $\boldsymbol{\vec{1}} \in \{1\}^{T\times 1}$.
        \item \textbf{Updating.} In updating step, the probability $\theta \in \mathbb{R}^{E\times K}$ is updated by using momentum with softmax function. We define a momentum matrix $V\in \mathbb{R}^{E\times K}$ at an iteration $l$ as
        \begin{equation}\label{eq:v_estimation}
            V_{l+1}[e,k] = \gamma V_l[e,k] + (1-\gamma) \frac{ e^{ \left(  \overline{P}[e,k] \right)}}{\sum_j e^{ \left( \overline{P}[e,k] \right)}}.
        \end{equation}
        We set $V_0$ as \textcolor{black}{a} zero matrix, and the corresponding new probability is obtained through
        \begin{equation}\label{eq:theta_estimation}
            \theta_{l+1} = \theta_{l} + V_{l+1},
        \end{equation}
        which is then normalized by softmax function to ensure the sum of each row is $1$.
    \end{itemize}
}
In a specific edge $(i,j)$, after the estimation process in Eq.~\ref{eq:theta_estimation}, 
the operation with the lowest probability is pruned by setting $\mathcal{O}^{(i,j)} = \mathcal{O}^{(i,j)}_{\setminus [k]} $, where $k = \argmin \theta^{(i,j)}$. 
It significantly reduces the search space from $2 \times \lvert\mathcal{O}^{(i,j)}\rvert^{14}$ to $2 \times (\lvert\mathcal{O}^{(i,j)}\rvert-1)^{14}$. 
To maintain the diversity of the architectures generated (details on the architecture generation would be introduced in Sec.~\ref{sec:3.3}), we consider reserving one operation in both parameter operation space $\mathcal{O}_p^{(i,j)}$ and non-parameter operation space $\mathcal{O}_n^{(i,j)}$ after the pruning process. Our dynamic distribution pruning algorithm is presented in Alg.~\ref{alg:MDL}.



\begin{algorithm}[t]
\caption{Dynamic Distribution Pruning }\label{alg:MDL}
\begin{algorithmic}[1]
\Require Training dataset; Validation dataset; Searching hyper graph $\mathcal{G}$
\Ensure Optimal probability distribution $\theta$
\While{$K \geq 2$}
  \While{$t\leq T$}
  \State Disjoint sample $l$ network architecture
  \State Inherit weights from $\mathcal{G}$
  \State Update weights using training data
  \State Evaluate on validation dataset
  \State $t = t+1$
  \EndWhile
  \State Estimation utility by Eq.~\ref{eq:info_reocrd} and Eq.~\ref{eq:utility_matrix}
  \State Update the distribution by Eq .\ref{eq:v_estimation} and Eq.~\ref{eq:theta_estimation}
  \For{$(i,j) \in \mathcal{E_M}$}
  \State $k = \argmin p(\theta^{(i,j)})$
  \State Prune the minimal $k$, $\mathcal{O}^{(i,j)} = \mathcal{O}^{(i,j)}_{\setminus [k]} $
  \EndFor
  \State $K = K-1$
\EndWhile
\end{algorithmic}
\end{algorithm}

\begin{theorem}\label{theorem:1}
In a certain training epoch $t$, the architecture variable $\alpha$ directly determines the validation performance, specifically: $p(y\vert\theta, t) \propto p(\alpha\vert\theta).$
\end{theorem}
\begin{proof}
{As illustrated in Fig.~\ref{fig:prob}}, the likelihood of the validation target for a given distribution parameter $\theta$ can be formulated as:
\begin{equation}\label{eq:objective}
\begin{aligned}
&p(y\vert\theta, t) = p(y\vert X_\text{train}, Y_\text{train}, W, X_\text{val},\theta, t) \\
&= p(\alpha \vert \theta)p(W\vert t,\alpha,X_\text{train},Y_\text{train})p(y\vert W,X_\text{val}),
\end{aligned}
\end{equation}
where $X_\text{train}$ and $Y_\text{train}$ are the inputs and labels from the training set, $W$ is the set of network weights, and $t$ denotes the training epochs. Since $X_\text{train}, X_\text{val}, Y_\text{train}$ are observed variables, during a specific training epoch $t$, Eq.~\ref{eq:objective} can be further simplified to:
\begin{equation}\label{eq:objective_short}
p(y\vert \theta, t) = p(\alpha\vert \theta)p(W\vert \alpha, t)p(y\vert W).
\end{equation}
To simplify the analysis, without loss of generality, we assume the network weights are initialized as constants, which means $W$ is fixed for a given architecture and a given training epoch. We can further simplify Eq.~\ref{eq:objective_short} to
\begin{equation}\label{eq:objective_shorter}
p(y\vert \theta, t) =p(\alpha\vert \theta)p(y\vert \alpha, t) \propto p(\alpha\vert \theta).
\end{equation}
As shown in Eq.~\ref{eq:objective_shorter}, the architecture variable $\alpha$ directly determines the validation performance, \emph{i.e.,} if an architecture shows better performance on the validation set, the corresponding $\theta^{(i,j)}_k$ holds a high probability, and vice versa. 
Therefore, Theorem \ref{theorem:1} is true during any specific training epoch.
\end{proof}

\subsection{Efficient Network Generation} \label{sec:3.3}

We obtain a fully-connected DAG after the pruning process Sec.~\ref{sec:3.2}. Each edge $(i,j)$ in the graph only contains one parameter operation and one non-parameter operation. 
{In previous research, together with the corresponding probability $\theta^{(i,j)}$, a discrete architecture is derived by maintaining $n=2$ strongest paths for each intermediate node, and thus every mixed operation is replaced with the most likely operation by taking the argmax, which is inflexible and uncontrollable \citep{chen2019progressive,xu2019pc,Zheng_2019_ICCV,liu2018darts} .}
Yu \emph{et.al.} \citep{yu2019network, cai2019once} proposed to evaluate the model on the validation set and greedily prune the operation with minimal drop in accuracy, which is computationally expensive. In contrast, the proposed efficient network generation (ENG) is flexible to arbitrary constraints and effective. The idea is intuitive: the final probability $\theta$ calculated by Eq.~\ref{eq:v_estimation} and Eq.~\ref{eq:theta_estimation} is a weighted accumulation of the validation performance through epoch $t$. 
Therefore, the optimal architecture can be obtained based on a new measure, which is called architecture expectation and defined as:
\begin{definition}{}
The architecture expectation is an average of the architecture $\alpha$ over the corresponding operation probability:
\begin{equation}
E[\alpha] = \frac{1}{\vert \mathcal{E_M}\vert}\sum_{(i,j),k} \alpha^{(i,j)}_k \theta^{(i,j)}_k,
\end{equation}
where $\vert\mathcal{E_M}\vert$ is the number of edges defined in Sec.~\ref{sec:3.1}.
\end{definition}
 The optimal architecture can be generated by:
\begin{equation}\label{eq:transformed_object}
\alpha^{*} = \argmax_\alpha E[\alpha] \; s.t. \; f(\alpha) \leq {\nu}, \alpha \in \mathcal{O}.
\end{equation}
where $f(.)$ is a function that calculates constraints for architectures, such as FLOPs and latency, and $\nu$ is given for different applications.
{Given $\nu$ for different deployed platforms, $f(.)$ is a function that estimates constraints for architectures (e.g. Latency and FLOPs).} 
{Optimizing Eq.~\ref{eq:transformed_object} is extremely efficient, since only a few minutes is demanded for traversing the whole constraint space $\mathcal{O}$ and calculating $E[\alpha]$ on a single CPU.}
For example, we want to generate architectures by constraining each intermediate node $B_i$ with two inputs in cell-based search space as \citep{liu2018darts,zoph2018learning}. Under these circumstances, there are only $30 \times 2^8 \approx 8 \times 10^3$ architectures in the pruned search space.

\subsection{Theoretical Analysis}\label{sec:3.4}

\begin{figure}
\centering
\includegraphics[width=1.0\linewidth]{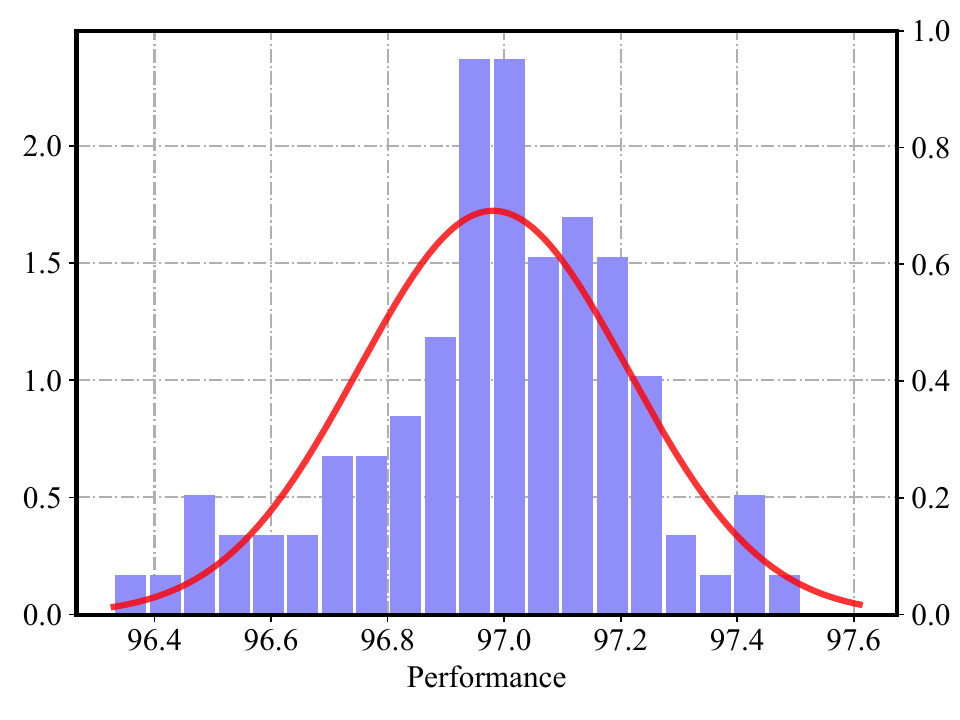}
\caption{\label{fig:pruning_mistake} The performance estimation distribution on a specific operation. In order to investigate the distribution of an operation, we fix the operation on \textcolor{black}{an} edge, then sample 100 networks in this pruned search space and train these network\textcolor{black}{s} until convergence. We can see that these networks roughly obey a Gaussian distribution.\label{fig:fig_2}}
\end{figure}

We further provide a theoretical bound for our performance pruning methods. As illustrated in Fig.~\ref{fig:fig_2}, we assume that $p(y\theta, t)$ follows a Gaussian distribution with $\mu$ as the expectation and additive Gaussian noise:
\begin{equation}
p(y\vert \theta, t) =\mathcal{N}(\mu_t, \sigma_t) = \mu_t + N(0, \epsilon_t) = \mu_t + \epsilon_t,
\end{equation}
where $\epsilon$ is a zero mean Gaussian random variable with a variance $\sigma_t$. A corollary is given to show that the distribution is bounded when converging. \footnote{The details about theoretical proof and the experiment of the Gaussian assumption are provided in the supplementary material.}
\begin{corollary} \label{corollary:1}
In the $t$-th epoch, the performance distribution is
\begin{equation}\label{eq:sigma_function}
p(y\vert \theta, t) \sim \mathcal{N} (\mu_t, \sigma_t),
\end{equation}
where $\sigma_t = \beta \Delta + \eta$, and $\beta, \eta$ are two constants. $\Delta = t^{*} -t$ where $t^* $ is the converging epoch.
\end{corollary}
Corollary~\ref{corollary:1} is a generalized conclusion from \citep{ying2019bench}, which shows that the standard deviation of the performance estimation will be bounded when the network converging. In \citep{ying2019bench}, the variance has an increasing Spearman Rank Correlation ($r_s$) when converging. Specifically, $r_s$ is linear to  $\Delta$ (0 for converging), and thus bounded. This assumption is true in learning rate reduction schemes, such as the cosine annealed schedule. We generalize this empirical formulation using formal mathematical language in Eq.~\ref{eq:sigma_function} by introducing a deviation function for the estimation error $\sigma$, \emph{e.g.,} a low rank correlation corresponds to a high deviation. 
To analyze the error bound, we further define what constitutes a pruning mistake based on the distribution hypothesis.

\begin{theorem}\label{theorem:2}
Assuming $\xi_{0,t} \sim \mathcal{N}(\mu_{0,t},\,\sigma_t^2)$ and  $\xi_{1,t} \sim \mathcal{N}(\mu_{1,t},\,\sigma_t^2)$ denote the performance estimation distribution of the worst and second worst operations at the $t$-th epoch, respectively. The probability that a pruning error occurs is
\begin{equation}\label{eq:delta_o}
p(\xi_{1,t} - \xi_{0,t} < 0),
\end{equation}
with $\xi_{1,t} - \xi_{0,t} \sim \mathcal{N}(\delta_t, \,2\sigma_t^2)$. $\delta_t = \frac{\zeta_t}{K} > 0 $ denotes the average distance with $\zeta$ as the difference between the maximum performance and minimum performance.
\end{theorem}
\begin{proof}
$ p(\xi_{1,t} - \xi_{0,t} < 0)$ denotes the probability that the penultimate operation will be pruned, which causes a pruning mistake due to the greedy strategy. Since we estimate different operations independently, the distribution $\xi_{1,t} - \xi_{0,t}$ is calculated through the random variable convolution
\begin{equation}
\small
\begin{split}
&\mathbb{E}(\xi_{1,t} - \xi_{0,t}) = \mu_{1,t} - \mu_{0,t},\\
& \sigma_{\xi_{1,t} - \xi_{0,t}} = \sqrt{\sigma_{\xi_{1,t} }^2 + \sigma_{\xi_{0,t}}^2} = \sqrt{2}\sigma_t.
\end{split}
\end{equation}
Together with the fact that $\xi_{1,t}$ and $\xi_{0,t}$ are two adjacent operations in the sampling process, the distance between $\mu_{1,t}$ and $\mu_{0,t}$ is $\delta_t $. And the distribution follows $\xi_{1,t} - \xi_{0,t} \sim \mathcal{N}(\delta_t, \,2\sigma_t^2)$.
\end{proof}

\begin{theorem}\label{theorem:3}
The upper bound of the error rate of Alg.~\ref{alg:MDL} at the $t$-th epoch is
\begin{equation}
2 \left(\frac{ K \left(\beta \Delta_t + \eta\right)}{\zeta_t } \right)^2.
\end{equation}
\end{theorem}

\begin{proof}

According to Theorem~\ref{theorem:2}, and Chebyshev’s inequality, we have
\begin{equation}
\small
\begin{split}
p(\xi_{1,t}-\xi_{0,t}<0) &=  p(\xi_{1,t}-\xi_{0,t} - \delta_t< - \delta_t)\\
&= \frac{1}{2} \, p(\vert \xi_{1,t}-\xi_{0,t} - \delta_t\vert > \delta_t)\\
&\leq \frac{1}{2}\left(\frac{\sqrt{2}\sigma_t}{\delta_t}\right)^2 = \left(\frac{\sigma_t}{\delta_t}\right)^2.
\end{split} 
\end{equation}
While the bound above is between the worst and second worst samples, if we consider a series of $K$ samples, the overall bound of the total error rate is
\begin{equation}
\small
\begin{split}
\sum_{n=1}^K \left( \frac{\sigma_t}{n \delta_t} \right) ^2 &= \left(\frac{\sigma_t}{ \delta_t} \right) ^2 \sum_{n=1}^K \frac{1}{n^2}\\
&< \left(\frac{\sigma_t}{ \delta_t} \right) ^2 \left( 1 +\sum_{n=2}^K \frac{1}{n(n-1)}\right)\\
&= \left(\frac{\sigma_t}{ \delta_t} \right) ^2 \left( 2 - \frac{1}{K}\right).
\end{split}
\end{equation}
Based on $\sigma_t$ and $\delta_t$ defined in Corollary~\ref{corollary:1} and Theorem~\ref{theorem:2}, the error bound can be further formulated as 
\begin{equation}
\small
\begin{split}
\left( 2 - \frac{1}{K}\right)\left(\frac{K \left(\beta \Delta_t+ \eta \right)}{\zeta_t} \right)^2 <2 \left(\frac{ K\left(\beta \Delta_t + \eta \right)}{\zeta_t} \right)^2.
\end{split}
\end{equation}
\end{proof}

In the above theorems, we quantitatively demonstrate the rationale behind the dynamic pruning design. The error bound is decided by $t$ and $\zeta$. On one hand, when the training begins, $\Delta$ is large, and we have to be conservative to avoid pruning the architecture too early. On the other hand, when $t$ gets closer to $t^*$, we can prune more aggressively with a guaranteed low risk of missing the optimal architecture. 

\textbf{Complexity Analysis.} Finding the optimal architecture in the search space using previous methods \citep{zoph2018learning,real2018regularized} with $N$ different constraints involves $\mathcal{O}(K^{2\vert \mathcal{E_M}\vert } \cdot N)$ computations. In contrast, the proposed method dynamically prunes the search space every $T$ epochs, and generates architectures that satisfy $N$ constraints within a few minutes. Therefore, the complexity of our method is
\begin{equation}
\mathcal{O}(T\times \sum_{k=2}^K k) = \mathcal{O}(T\times \frac{(K+2)(K-1)}{2}) = \mathcal{O}(TK^2).
\end{equation}

\section{Experiments}\label{sec:experiment}

{In this section, we compare the proposed method against the state-of-the-art methods in effectiveness and efficiency on CIFAR10 and ImageNet. }
First, we conduct experiments under the same settings {following} \citep{liu2018darts,cai2018path,zoph2018learning,liu2018progressive} to evaluate the generalization capability, \emph{i.e.,} first searching on the CIFAR10 dataset, then stacking the optimal cells \textcolor{black}{into} deeper networks. Second, we further perform experiments of searching architectures directly on ImageNet under the mobile settings, following \citep{cai2018proxylessnas}. 
\begin{table*}[t]
\small
\begin{center}
\setlength{\tabcolsep}{2.5mm}{
\begin{tabular}{lcccc}
\toprule[1pt]
\multirow{2}{*}{\textbf{Architecture}} & \textbf{Test Error} & \textbf{Params} & \textbf{Search Cost} & \textbf{Search} \\
& \textbf{(\%)} & \textbf{(M)} & \textbf{(GPU days)} & \textbf{Method} \\ \hline
ResNet-18 \citep{he2016deep} & 3.53 & 11.1 & - & Manual \\
DenseNet \citep{huang2017densely} & 4.77 & \textbf{1.0} & - & Manual \\
SENet \citep{hu2018squeeze} & 4.05 & 11.2 & - & Manual \\ \hline
NASNet-A \citep{zoph2018learning} & 2.65 & 3.3 & 1800 & RL \\
AmoebaNet-A \citep{real2018regularized} & 3.34 & 3.2 & 3150 & Evolution \\
PNAS \citep{liu2018progressive} & 3.41 & 3.2 & 225 & SMBO \\
Random Search \citep{li2019random} & $3.03 \pm 0.13$ & 2.2 & 2.7 & Random Search \\
$\text{Random Search 100}^*$& $2.55 $ & 2.9 & 108 & Random Search \\
 $\text{DARTS}^*$ \citep{liu2018darts} & $2.7\pm0.01$ & 3.1 & 1.5 & Gradient-based \\
GDAS \citep{Dong_2019_CVPR} & 2.93 $\pm$ 0.07 & 3.4 & 0.8 & Gradient-based \\
 $\text{MdeNAS}^*$\citep{Zheng_2019_ICCV} & 2.80 $\pm$ 0.15 & 3.6 & 0.16 & MDL \\
$\text{Pdarts}^*$\citep{chen2019progressive} & 2.91 $\pm$ 0.06 & 3.4 & 0.3 & Gradient-based \\
$\text{PCdarts}^*$ \citep{xu2019pc} & 2.73 $\pm$ 0.01 & 3.6 & 0.1 & Gradient-based \\
{XNAS \citep{nayman2019xnas}} &{$2.55$} & {$3.79$}  &{0.3}  & {Gradient-based }      \\
{GAGE \citep{li2020geometry}}  &{$2.67$} & {3.63 } &{0.3 } &{ Gradient-based  }     \\
{DARTS- \citep{chu2020darts}}  &{$2.59\pm 0.08$} & {3.5}  &{0.4 } & {Gradient-based }     \\
{MIGO-NAS \citep{zheng2021migo}}  &{$2.65\pm 0.14$} & {$3.23\pm0.2$}  &{0.06} & {MIGO}     \\
\hline
\textbf{DDPNAS} & \textbf{$2.59 \pm 0.01 \vert 2.44$} & $3.16$ & \textbf{0.075} & \textbf{Pruning} \\
\textbf{DDPNAS NoPrune} & $2.75 \pm 0.03 \vert 2.70$ & $2.1$ & 0.25 & - \\

 \bottomrule[1pt]
\end{tabular}}
\end{center}
\caption{ Test error rates ($\mu \pm \sigma\vert\min$) for our discovered architectures, human-designed networks and other NAS architectures on CIFAR10. For a fair comparison, we train the architectures with similar parameters ($<$ $5$M) and exactly \emph{the same training conditions}. $*$ denotes that we search the architecture with the provided code. ``Random Search 100" denotes that we randomly choose $100$ architectures, to select the optimal architecture using standard training and evaluation. ``DDPNAS NoPrune" denotes that we do not prune operators in the search process.}
\label{tab:cifar_results}
\end{table*}

\begin{itemize}
    \item {Regarding the search process, the ofﬁcial training images are \textcolor{black}{randomly} split into two groups, with each group containing 25K images in CIFAR-10. The former is assigned to the training set $X_\text{train}, Y_\text{train}$ in Alg.~\ref{alg:MDL}, and the other is used as the validation set $X_\text{val}, Y_\text{val}$ in Alg.~\ref{alg:MDL}. After disjointly sampling $K$ architectures, their corresponding weights are inherited from one-shot model $\mathcal{G}$. The sampled network \textcolor{black}{is} trained with $X_\text{train}, Y_\text{train}$ and evaluated on $X_\text{val}, Y_\text{val}$. After repeating the previous steps for $T$ times, we update the parameters of the distribution $\theta$ \textcolor{black}{depending} on the performance of the sampled architectures by Eq.~\ref{eq:info_reocrd}, Eq.~\ref{eq:utility_matrix} and Eq.~\ref{eq:v_estimation}. Then the operation is pruned for every edge with the minimum probability. The searching is terminated when the search space retains one operation (\emph{i.e., $K = 1$}).}
    \item After the searching process, we obtain a distribution $\theta$ which is prone to sample architectures with high performance. In terms of cell-based search space~\citep{liu2018darts}, we follow \citep{liu2018darts} to derive discrete architectures: selecting two operations with the largest probability from $\theta$ in each intermediate node to generate the architecture. However, the proposed ENG is also capable of generating architecture with different heights and widths, where the experiment is conducted in Tab.~\ref{tab:different_constraint}. When the search space is chain-structured, we employ the proposed efficient network generation according to the distribution to select the architectures under different hardware constraints. Both architectures are then trained from scratch to obtain the final performance. 
\end{itemize}

\textbf{Experimental Settings.} We use the same datasets and evaluation metrics as existing NAS works \citep{liu2018darts,cai2018path,zoph2018learning,liu2018progressive}. First, most experiments are conducted on CIFAR10, which has $50$K training images and $10$K testing images with resolution $32 \times 32$, from $10$ classes. The color intensities of all images are normalized to $[-1, +1]$. During the architecture search, we randomly select $5$K images from the training set as a validation set. To further evaluate the generalization capability, we stack the optimal cells discovered on CIFAR10 into a deeper network, and then evaluate the classification accuracy on ILSVRC 2012 ImageNet, which consists of $1,000$ classes with $1.28$M training images and $50$K validation images. Here, we consider the \emph{mobile} setting, where the input image size is $224 \times 224$ and the number of multiply-add operations is less than 600M.

In the search process, we consider a total of $L=6$ cells in the network, where the reduction cells are inserted in the second and the third layers, with $M = 4$ nodes in each cell. The search epoch correlates to the estimating epoch $T$. In our experiment, we set $T = 3$, so the network is trained for less than $100$ epochs, with a batch size of $512$, and $16$ initial channels. We use SGD to optimize the network weights $W$, with an initial learning rate of $0.025$ ({annealed down to zero following a Cosine schedule}), a momentum of 0.9, and a weight decay of $3 \times 10^{-4}$. The learning rate of category parameters is set to $0.01$. The search takes only $1.8$ GPU hours with one Tesla V100 on CIFAR10.

\begin{figure*}[htb]
\centering
\includegraphics[width=1.0\linewidth]{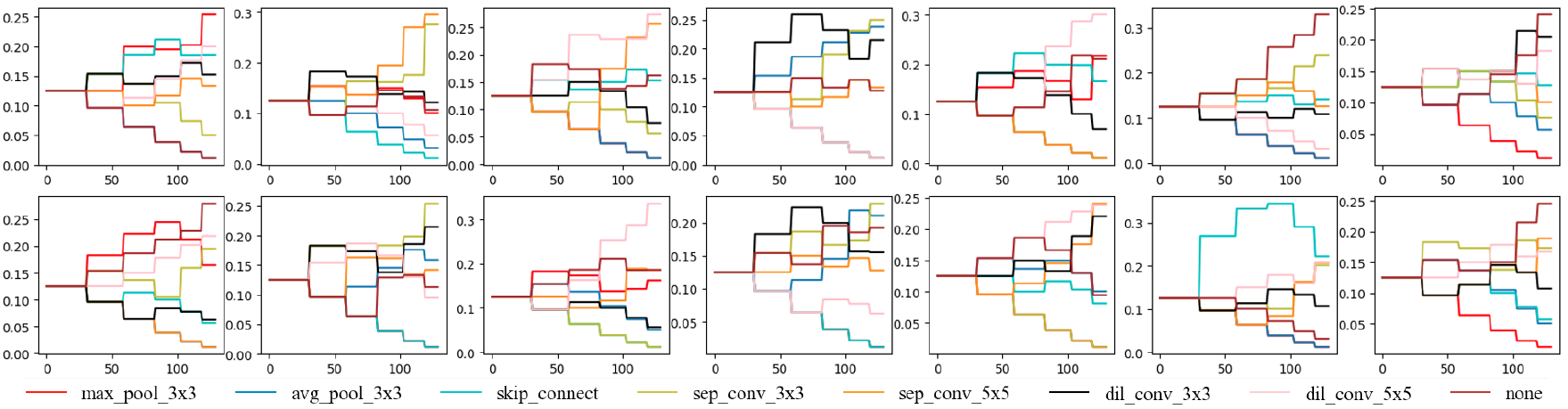}
\caption{\label{fig:weight} {The probability of each operation and edge in the DARTS search space in the search process. Each subgraph represents a specific edge in the search space, and the lines in each subgraph represent different operations on the edge. Note that during the search process, we retain the corresponding probability for the pruned operation for normalization.}}
\end{figure*}

\begin{figure}[htb]
\centering
\begin{subfigure}[b]{0.5\textwidth}
    \centering
   \includegraphics[width=1.0\linewidth]{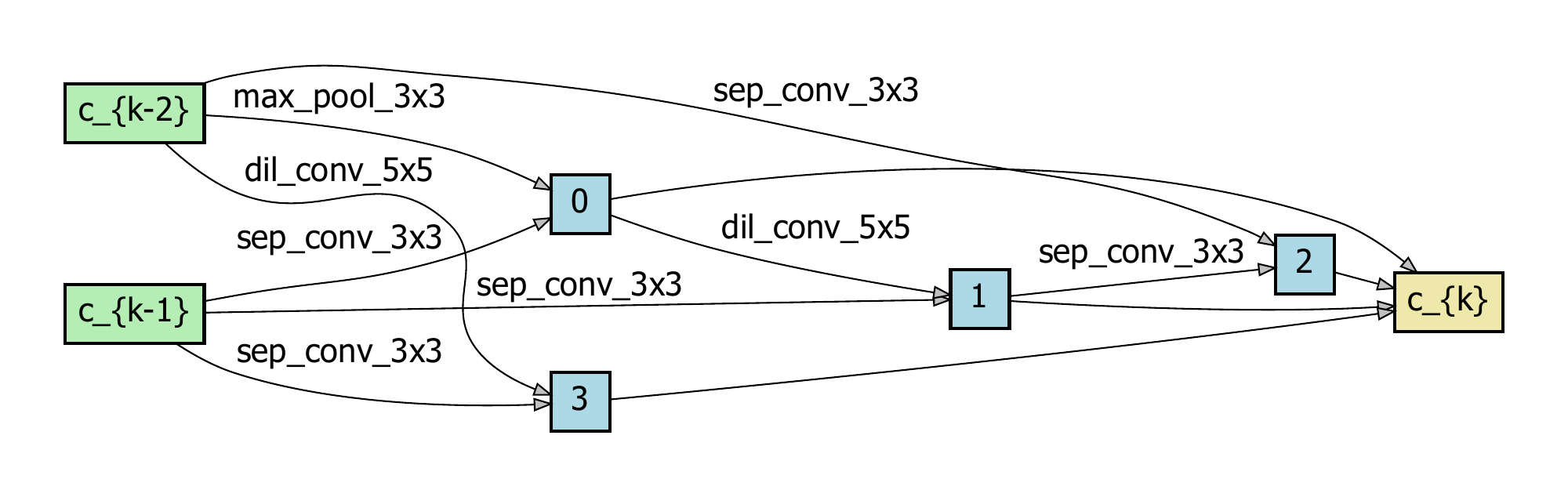}
   \caption{Normal cell}
   \label{fig:Ng1} 
\end{subfigure}
\newline
\newline
\begin{subfigure}[b]{0.5\textwidth}
\centering
   \includegraphics[width=1.0\linewidth]{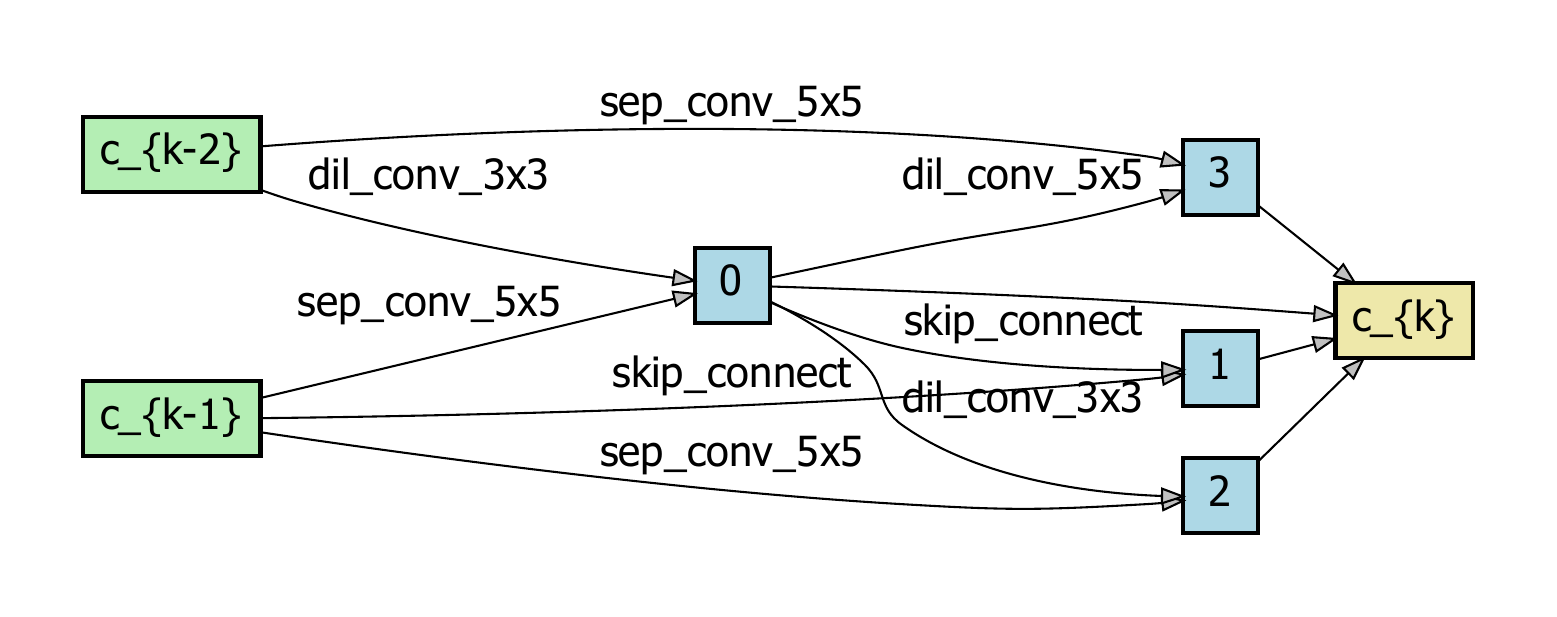}
   \caption{Reduction cell}
   \label{fig:Ng2}
\end{subfigure}

\caption{\label{fig:darts_architectures} {The found normal cell and reduction cell by the proposed method.}}
\end{figure}

In the architecture evaluation step, our experimental settings are similar to \citep{liu2018darts,zoph2018learning,pham2018efficient}. A large network of $20$ cells is trained for $600$ epochs with a batch size of $96$ and additional regularization, such as cutout \citep{devries2017improved}. When stacking cells to evaluate on ImageNet, we use two initial convolutional layers of stride $2$ before stacking $14$ cells with scale reduction at the $1$st, $2$nd, $6$th and $10$th cells. 
The network is trained for 250 epochs with a batch size of 512, a weight decay of $3 \times 10 ^{-5}$, and an initial learning rate of 0.1. All the experiments and models are implemented in PyTorch \citep{paszke2017automatic}.

In this experimental setting, we first search neural architectures on an over-parameterized network, and then evaluate the best architecture with a stacked deeper network. We run the experiment multiple times and find that the resulting architectures only show \textcolor{black}{a} slight variance in performance, which demonstrates the stability of the proposed method.

\begin{table*}[t]
\begin{center}
\setlength{\tabcolsep}{0.5mm}{
\begin{tabular}{c|cccc|cc}
\hline
\multirow{2}{*}{\textbf{Model}} & \textbf{Search}    & \textbf{Search}     & \textbf{Search Cost} & \textbf{Search}   & \multirow{2}{*}{ \textbf{FLOPs} } &\textbf{Top1}    \\
                       & \textbf{method}    & \textbf{Space}      & \textbf{(GPU days)}   & \textbf{dataset}  &                        & \textbf{acc}\\ \hline
1.0-MobileNetV2 \citep{sandler2018mobilenetv2}        & manual    & -          & -           & -        & 300M                   & 72.0    \\
1.5-ShuffleNetV2 \citep{ma2018shufflenet}       & manual    & -          & -           & -        & 299M                   & 72.6    \\
1.4-MobileNetV2  \citep{sandler2018mobilenetv2}       & manual    & -          & -           & -        & 585M                   & 74.7    \\
2.0-ShuffleNetV2  \citep{ma2018shufflenet}       & manual    & -          & -           & -        & 591M                   & 74.9    \\ \hline
DARTS      \citep{liu2018darts}            & gradient  & cell-wise       & 4           & CIFAR-10 & 574M                   & 73.3    \\
GDAS      \citep{Dong_2019_CVPR}            & gradient  & cell-wise       & 0.21           & CIFAR-10 & 581M                   & 74.0    \\
NASNet-A   \citep{zoph2018learning}        & RL        & cell-wise       & 1800        & CIFAR-10 & 564M                   & 74.0    \\
AmoebaNet-C \citep{real2018regularized}            & evolution & cell-wise       & 3150        & CIFAR-10 & 570M                   & 75.7    \\
PNAS  \citep{liu2018progressive}                 & SMBO      & cell       & 225         & CIFAR-10 & 588M                   & 74.2    \\
{FAIR-Darts  \citep{chu2020fair}}        & {gradient}  & {cell-wise} & {0.4}           & {CIFAR-10} & {541M}                  & {75.1}   \\
{RL-NAS  \citep{zhang2021neural}}        & {unsup}  & {cell-wise} & {-}           & {CIFAR-10} & {581M}                  & {75.6}   \\
{RL-NAS  \citep{zhang2021neural}}        & {unsup}  & {cell-wise} & {-}           & {ImageNet} & {561M}                  & {75.9} \\
DDPNAS                 & pruning   & cell-wise       & 0.016       & CIFAR-10 & 450M                   & 75.2       \\ \hline
MnasNet-92 \citep{tan2019mnasnet}         & RL  & cell-wise & 3000+           & ImageNet & 388M                   & 75.3   \\
FBnet-C \citep{wu2019fbnet}            & gradient  & layer-wise & 9           & ImageNet & 375M                   & 74.9   \\
$\text{proxyless-GPU}$  \citep{cai2018proxylessnas}        & gradient  & layer-wise & 4           & ImageNet & 465.M                  & 75.1   \\
$\text{proxyless-CPU}$  \citep{cai2018proxylessnas}        & gradient  & layer-wise & 4           & ImageNet & 439 M                  & 75.3   \\
{RL-NAS  \citep{zhang2021neural}}        & {unsup}  & {layer-wise} & {-}           & {ImageNet} & {473M}                  & {75.6} \\
DDPNAS-A              & pruning   & layer-wise & 2           & ImageNet & 599M               & 77.2       \\
DDPNAS-B              & pruning   & layer-wise & 0           & ImageNet & 414M                      & 75.5       \\
DDPNAS-C              & pruning   & layer-wise & 0           & ImageNet & 211M                      & 74.2       \\ \hline
\end{tabular}}
\end{center}
\caption{Comparison with the state-of-the-art image classification networks on ImageNet under the mobile settings. We divide these baseline methods into $3$ groups according to the search dataset for a fair comparison. Please note that, we only need to search once to generate architectures (DDPNAS-A/B/C) under different $\nu = {600, 400,200}$ FLOPs with the proposed ENG in Sec.~\ref{sec:3.3}. Therefore, the search time of DDPNAS-B and DDPNAS-C is 0. $*$ denotes that these baseline models are trained with our code exactly for the same training condition. }
\label{tab:_imagenet}
\end{table*}

\subsection{Results on CIFAR10}

We compare our method with both manually designed and NAS networks. The manually designed networks include ResNet \citep{he2016deep}, DenseNet \citep{huang2017densely} and SENet \citep{hu2018squeeze}. For NAS networks, we compare our approaches against various search methods, including RL methods \citep{zoph2018learning}, evolutional algorithms \citep{real2018regularized}, Sequential Model Based Optimization (SMBO) \citep{liu2018progressive}, gradient-based methods \citep{liu2018darts,Dong_2019_CVPR}, random search \citep{li2019random} and multinomial distribution learning \citep{Zheng_2019_ICCV}.

The summarized results for convolutional architectures on CIFAR10 are presented in Tab.~\ref{tab:cifar_results}. It is worth noting that the proposed method outperforms the state-of-the-arts \citep{zoph2018learning,liu2018darts} in accuracy, with much lower computational consumption. We attribute our superior results to our novel way of solving the problem with pruning, as well as the fast learning procedure: the network architecture can be directly obtained from the distribution when it converges. In \textcolor{black}{contrast}, previous methods \citep{zoph2018learning} evaluate architectures only when the training process is complete, which is highly inefficient. 
Another notable observation from Tab.~\ref{tab:cifar_results} is that, even with random sampling in the search space, the test error rate in \citep{liu2018darts} is only $3.03$\%, which is comparable with the previous methods in the same search space. We can conclude that the high performance of the previous methods is partially due to the search space that is dedicatedly and manually designed with specific expert knowledge. 
Meanwhile, the proposed method quickly explores the search space and generates a better architecture. We also report the results of hand-crafted networks in Tab.~\ref{tab:cifar_results}. Clearly, our method shows a notable enhancement, which indicates its superiority in both resource consumption and test accuracy. 

{Fig.~\ref{fig:weight} \textcolor{black}{shows} the \textcolor{black}{changing process} of the distribution in the search process. Note that during the search process, we retain the corresponding probability for the pruned operation for normalization. Therefore, all the probabilities keep changing in the whole search in Fig.~\ref{fig:weight}.
Interestingly, the ``model collapse"\footnote{According to the previous work \citep{liang2019darts+},  after certain search epochs, the number of skip-connects increases dramatically in the selected architecture, which results in poor performance. } that \textcolor{black}{is} considered as a key factor that influences the final performance of the searched architecture did not appear during the search of DDPNAS. Conversely, many skip connections are pruned in the \textcolor{black}{beginning} ([row1, col2], [row2, col2], [row2, col4]). 
Meanwhile, with the search progresses, more and more skip connections are pruned, and parameter operations are preserved. Finally, from the searched architecture illustrated in Fig.~\ref{fig:darts_architectures}, we can see that most of the edges 
prone to operations with parameters such as sep\_conv\_$3\times3$ and dil\_conv\_$5\times5$. 
A possible explanation is that selecting operations with parameters will result in a more complex neural architecture \textcolor{black}{that} yields better performance eventually. 
Another interesting observation in the searched architecture is the dominance of sep\_conv\_$3\times3$. 
According to the released codes \footnote{\url{https://github.com/quark0/darts/blob/master/cnn/operations.py}}, we find that sep\_conv\_3x3 is deeper than dil\_conv, which could increase the complexity of the architecture and lead to better performance. 
In conclusion, we should design a complex model as much as possible to improve performance. }

\begin{table}[t]
\begin{center}
\setlength{\tabcolsep}{3.5mm}{
\begin{tabular}{lcc}
\hline
 \multirow{2}{*}{\textbf{Constraints $\nu$}}  & \textbf{Params} & \textbf{Test Error} \\
  & \textbf{(M)} & \textbf{(\%)}
 \\
\hline
$\{N = 2\}$ & 3.57 $\pm$ 0.02  & 2.72 $\pm$ 0.005 \\
$\{N = 4\}$ & 3.16 $\pm$ 0.01  & 2.59 $\pm$ 0.01 \\
$\{N = 6\}$ & 2.83 $\pm$ 0.48 &  2.73 $\pm$ 0.02\\
$\{H = 1, N = 4\}$ & 2.8 $\pm$ 0.06 &  2.93 $\pm$ 0.03\\
$\{H = 2, N = 4\}$ & 2.7 $\pm$ 0.05  & 2.87 $\pm$ 0.08 \\
$\{H = 3, N = 4\}$ &  2.7 $\pm$ 0.10 &  2.70 $\pm$ 0.04\\
\hline
\end{tabular}}
\end{center}
\caption{Mean and variance of CIFARF-10 test error and model size under various $\nu=\{H,N\}$ in $4$ runs. $H$ and $N$ denote the height and the number of non-parameter constraints, respectively. }
\label{tab:different_constraint}
\end{table}

\subsection{Result on ImageNet and Transferring}
In Tab.\ref{tab:_imagenet}, we first compare our method under the mobile settings on ImageNet to demonstrate the generalization capability. 
The best architecture obtained by our algorithm on CIFAR10 is transferred to ImageNet, which follows the same settings as the works in \citep{zoph2018learning,pham2018efficient,cai2018path}. 
{The proposed method achieves comparable \textcolor{black}{performance} to the state-of-the-art methods with far less computational cost~\citep{zoph2018learning,real2018regularized,liu2018progressive,real2018regularized,liu2018progressive,pham2018efficient,liu2018darts,cai2018path}.}

{The minimal time and GPU memory consumption make applying our algorithm on ImageNet feasible. The search experiment is \textcolor{black}{further} conducted with Layer-wise search space on ImageNet. On ImageNet, the same search hyper-parameters \textcolor{black}{are} utilized as on CIFAR10.} We follow training settings in \citep{cai2018proxylessnas}, training the models for $120$ epochs with a learning rate of $0.4$ (annealed down to zero following a cosine schedule), and a batch size of $1024$ across four Tesla V100 GPUs. Experimental results are reported in Tab.~\ref{tab:_imagenet}, where our DDPNAS achieves superior performance compared to both human-designed and automatically searched architectures, with much lower computational cost. 
{The architectures could be generated directly under different constraints, by only searching once with the proposed ENG.}
Therefore, there is no extra search time for DDPNAS-B and DDPNAS-C.

\begin{figure}[t]
\center
\includegraphics[width=1.0\linewidth]{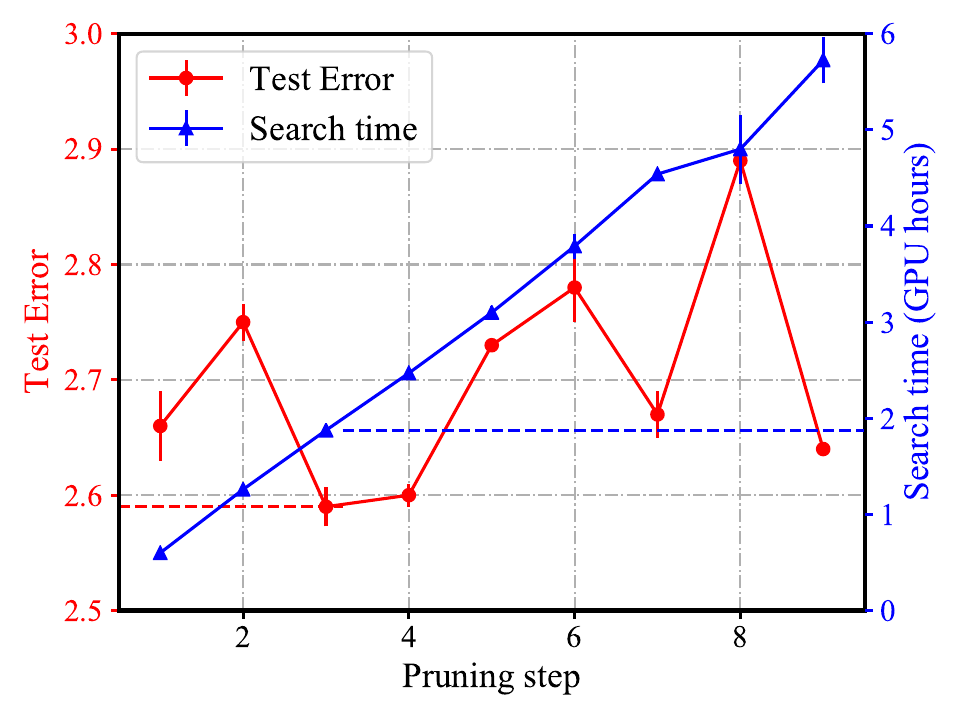}
\caption{\label{fig:pruning_step} The effectiveness of the proposed method with different pruning \textcolor{black}{steps } (with $4$ runs) in terms of test error and search time. The small dots denote \textcolor{black}{the} mean value and the column represents the standard deviation. The dash\textcolor{black}{ed} lines denote the choice of the pruning step in our paper.}
\end{figure}

\begin{figure}[t]
\includegraphics[width=1\linewidth]{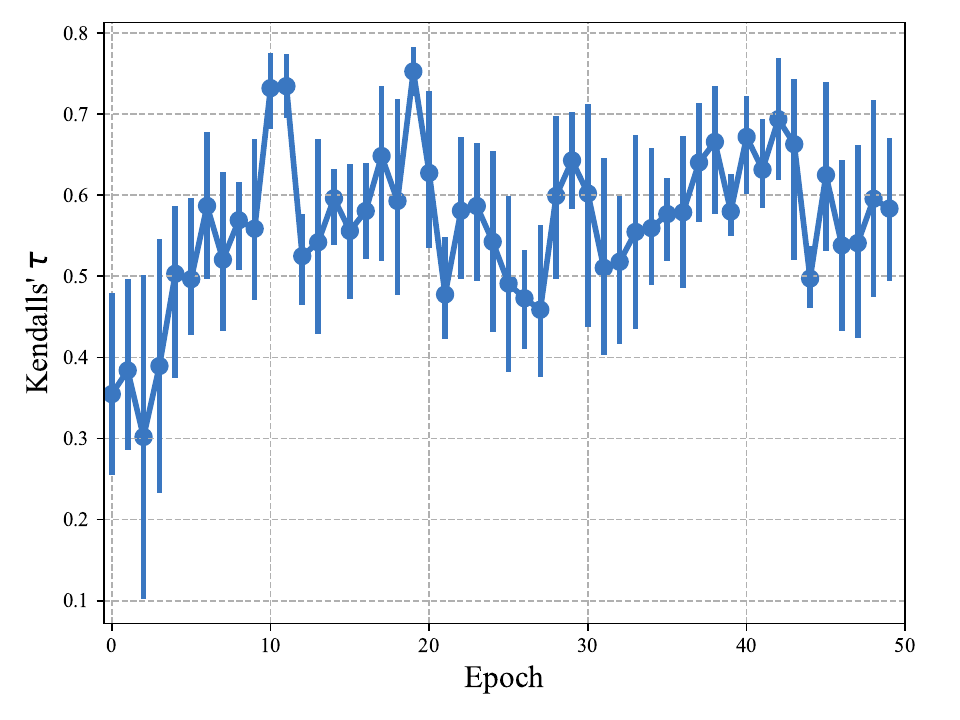}
\caption{{Performance estimation of different epochs in DDPNAS. All the architectures are sampled in NAS-Bench-201. And Kendall's $\tau$ \textcolor{black}{is calculated} in four different random seeds, where the mean and std are reported in the figure. There \textcolor{black}{is} always a high correlation (around $0.5722$) with the truth performance in different training epoch\textcolor{black}{s}. }}
\label{fig:pe_epoch}
\end{figure}

\subsection{Ablation Study}
\textbf{Results under Different Constraints.}
Tab.~\ref{tab:different_constraint} demonstrates the results under different constraints with cell based search space on CIFAR-10. Benefiting from the ENG method, we can generate specialized neural networks for all scenarios without additional searching cost, while previous methods \citep{xu2019pc,chen2019progressive} typically require explicit searching with corresponding \emph{human prior} regularization terms. Specifically, in Tab.~\ref{tab:different_constraint}, from \emph{one search process}, we can control a group of networks with different heights, widths, and model sizes, with a negligible drop in performance. In addition, we can also generate networks (DDPNAS-A/B/C) under different FLOPs constraints ($600/400/200$M) in Tab.~\ref{tab:_imagenet}, which also shows a state-of-the-art performance with similar FLOPs comparing with other models.

\begin{table*}[htb]
\small
\setlength{\tabcolsep}{1 mm}
\centering
\begin{tabular}{c|c|cccccccc}
\hline
\textbf{Method} & \textbf{Params}    & \textbf{SST-2}     & \textbf{MRPC} & \textbf{CoLA}   &  \textbf{RTE} & \textbf{MNLI} & \textbf{QQP} & \textbf{QNLI} & \textbf{Average} \\ \hline
$\text{BERT}_{12}$    & 110m    & 92.7          & 89.5           & 54.3        & 71.1        &   83.5    &    89.8       &  91.2  &  81.7\\
\hline
LayerDrop \citep{fan2019reducing} & 67m & 90.7 & 85.9 & 45.4 & 65.2 & 80.7 & 88.3 & 88.4  & 77.8 \\
DistilBERT \citep{sanh2019distilbert} & 67m & 91.3 & 87.5 & 51.3 & 59.9 & 82.2 & 88.5 & 89.2 & 78.6 \\
BERT-PKD \citep{sun2019patient}  & 67m & 91.3 & 85.7 & 45.5 & 66.5 & 81.3 & 88.4  & 88.4 & 78.2\\
PD-BERT \citep{pd} & 67m & 91.1 & 87.2 & - & 66.7 & 82.5 & 89.1 & 89.0 & - \\
\hline
\textbf{DDPNAS} & \textbf{59m$\sim$67m} & \textbf{91.8} & \textbf{87.7} & \textbf{54.1} & \textbf{69.3} & \textbf{82.9} & \textbf{90.2}  & \textbf{86.5}  & \textbf{80.4} \\
\textbf{DDPNAS NoPruning} & \textbf{59m$\sim$67m} & \textbf{90.8} & \textbf{86.7} & \textbf{52.1} & \textbf{68.7} & \textbf{82.4} & \textbf{88.5}  & \textbf{85.4}  & \textbf{79.2} \\
\hline
\end{tabular}
\caption{\textcolor{black}{Comparison} of DDPNAS with previous works under the same constraints on the development set of GLUE.\label{tab:BERT-compression}}
\end{table*}

 \textbf{Effectiveness of Pruning.}
Tab.~\ref{tab:cifar_results} and Tab.~\ref{tab:BERT-compression} also compare the result \emph{without pruning}, which shows a worse performance in search cost and classification error. To further explore this problem, we also fine-tune the pruning step in Fig.~\ref{fig:pruning_step}. We found that there is a test error decrease with a pruning step between $[2, 4]$, which controls the trade-off between performance and efficiency. We also found that the algorithm tends to select non-parameter layers (max/average pooling, skip connection) with a pruning step larger than $4$, causing an unstable performance in Fig.~\ref{fig:pruning_step}. The above observations are consistent with previous \textcolor{black}{works} \citep{xu2019pc,liu2018progressive} that the model will be overfitting in searching. \textcolor{black}{These authors} propose to use a kind of ``human prior" to enhance the performance. However, in our paper, this problem is simply solved by setting the pruning step less than $4$.

{\textbf{Performance estimation. }We conduct a new experiment to validate the effectiveness of our weight sharing strategy. As illustrated in Fig. \ref{fig:pe_epoch}, there is always a high correlation (around $0.5722$) with the truth performance in different training epoch\textcolor{black}{s}. Here, we use a toy example to explain the effect of performance estimation in DDPNAS. Specifically, we model the performance of the $K$ sampled network architectures using $K$ Gaussian distributions $\mathcal{N}_1(\mu_1, \sigma_1)$, $\mathcal{N}_2(\mu_2, \sigma_2)$, $...\mathcal{N}_K(\mu_K, \sigma_K)$. Meanwhile, we further assume that the $K$ Gaussian distributions share the same variance $\sigma_1 = \sigma_2, ..., \sigma_K = \sigma$ and the mean values are decreased sequentially with the same distance $\mu_1 - \mu_2 = \mu_2 - \mu_3=... =\mu_{K-1} - \mu_K = C\sigma$. In this way, we can adjust the magnitude of $C$ to obtain different qualities of the evaluation. Specifically, when $C$ is small, the examples sampled from the $K$ Gaussian distributions are prone to change in position, resulting in poor evaluation quality, and vice versa. In general, when Kendall's $\tau = 0.57$, the corresponding $C$ is approximately equal to $0.41$. In this case, the probability that we incorrectly prune the best operation is $\mathbb{P}\left(X<\mu+K\cdot C \sigma \right) = \mathbb{P}\left(X<\mu+ 3.28 \sigma \right)>0.99$. In addition, we propose to collect the results in $T$ epochs for evaluation, \textcolor{black}{incorporating} a momentum term, which further \textcolor{black}{enhances} the quality.}

\textbf{Influence of $\gamma$.} From the tuning of $\gamma$ in our experiment, we found that there is no significant difference in terms of $\gamma$ in $[0.7, 0.9]$, which is consistent with previous works \citep{qian1999momentum}.

\subsection{{Transferability}}
{To further validate the transferability of DDPNAS and the searched architectures to diverse tasks in the wild, we conduct two extensive experiments on NLP and semantic segmentation tasks. We first employ DDPNAS on a BERT-compression search space to validate the generalizability of different search spaces. Then, we embed our discovered architectures as backbones into semantic segmentation to evaluate the generalizability of the discovered architectures. We achieve comparable performance to other methods on both of these two tasks, which proves that the DDPNAS can be well transferred to other tasks.}

\subsubsection{BERT-compression}
\textbf{Search space setups.} The search space of BERT-compression consists of two levels. The first level is searching for the different layer numbers $d\in\{6, 8, 10, 12\}$. In each layer, we further allow the intermediate layer $k \in \{512,768,1024,3072\}$ and the head number $h \in \{4,8,12\}$ to be searchable. In this case, the search space is huge enough ($10^{13}$) to support any compression requirements.

{
\textbf{Datasets and metrics.} The search architectures are evaluated on the widely-used GLUE benchmark \citep{wang2018glue}, which consists of $8$ datasets and tested using different metrics. Specifically, SST-2, QNLI QQP and RTE use accuracy as the evaluation metric. The average result of MNLI-m and MNLI-mm is reported on MNLI. While the CoLA is evaluated on Matthew's correlation. And the MRPC is evaluated on F1 and accuracy. Following BERT (Devlin, et al. 2019), we do not consider the controversial WNLI dataset.}

\begin{table*}[tb]
\centering
\begin{tabular}{lccccc}  
\hline
Models&Backbone&mIoU($\%$)&Parameters(M)&FLOPs(B) \\ 
\hline
Resnet-50 ASPP$^{*}$      &  Resnet 50   &   $73.86$        & $23.05$           & $308.74$     \\ 
MobilenetV3 ASPP$^{*}$ &MobileNetV3  & $66.77$ & $6.7$ & $4.45$   \\ 
ESPNetV2 \citep{mehta2019espnetv2} &-&$62.1$&-&$3.86$\\
ESPNetV1 \citep{mehta2018espnet} &-&$60.3$&-&$5.09$\\
CCC2 \citep{park2018concentrated} &-&$62$&$0.192$&$6.29$\\
ALGNet \citep{zhou2020aglnet} &-&$70.1$&$1.12$ & $13.8$\\
\hline
DDPNAS &DDPNAS&$71.6$&3.01&$5.2$\\
\hline
\end{tabular}  
\caption{\textcolor{black}{Comparison} of DDPNAS with previous semantic segmentation works on Cityscapes. $^{*}$ denotes that these baseline models are trained using our implementation with the same training conditions.}  
\label{tab:segmentation}
\end{table*}

{\textbf{Results.} The performance and parameters of DDPNAS are compared with existing BERT compression methods, including BERT-PKD \citep{sun2019patient}, DistilBERT \citep{sanh2019distilbert}, PD-BERT \citep{pd},  and LayerDrop \citep{fan2019reducing}. As reported in Tab.~\ref{tab:BERT-compression}, DDPNAS outperforms all previous STOA models compression methods  with the same level constraint (67M). It is noteworthy that, compared with LayerDrop, DDPNAS achieves a $4.1\%$ higher accuracy on RTE and a $8.7\%$  improvement on CoLA . For DistilBERT \citep{sanh2019distilbert} in which large external data is used, there is still a $1.8\%$ improvement of average accuracy.}

\subsubsection{{Architecture Transfer on Semantic Segmentation}}
{We further validate the \textcolor{black}{transferability} of our method on Semantic segmentation tasks, in which a dense labeling map of object categories is provided for input images. The discovered DDPNAS architecture is directly inserted into atrous spatial pyramid pooling \citep{chen2017deeplab} to robustly segment objects at multiple scales. We \textcolor{black}{conduct} all the experiments on the Cityscapes \citep{Cordts2016Cityscapes}, which contains high-resolution images of $1,024 \times 2,048$ resolutions with pixel-wise annotations. $5, 000$ finely annotated images collected from $50$ cities are included in the \textcolor{black}{dataset}, in which $2, 975$, $500$ and $1, 525$ are split for training, validation and testing respectively. All our architectures are trained independently without pretraining on ImageNet. According to the evaluation protocol, we \textcolor{black}{utilize} $19$ outputs of $30$ semantic labels for evaluation. } 

{Tab.~\ref{tab:segmentation} reports the results on Cityscapes, the architecture equipped with DDPNAS as backbone achieves superior performance than human design ESPNetV2 ~\citep{mehta2019espnetv2}, ESPNetV1~\citep{mehta2018espnet}, CCC2~\citep{park2018concentrated} and ALGNet~\citep{zhou2020aglnet} by $9.5$, $11.3$ $9.6$ and $1.5$, respectively. Meanwhile, the FLOPs are only $5.2$B, which is more efficient than CCC2~\citep{park2018concentrated} and ALGNet~\citep{zhou2020aglnet}. Besides, the found DDPNAS architectures also achieve a superior trade-off between model complexity and mIOU. All of our models outperform by at least $4.83$ with similar FLOPs, compared to MobilenetV3 \citep{howard2019searching}. }

\section{Conclusion}
In this paper, we presented DDPNAS, the first pruning-based NAS algorithm, which reduces the search time by pruning the search space in the early training stage.
Together with the efficient network generation, the proposed framework can drastically reduce the computational cost under different constraints, while achieving excellent model accuracies on CIFAR10 and ImageNet. Furthermore, DDPNAS can directly search on ImageNet, outperforming human-designed networks and other NAS methods under mobile settings.

\section{acknowledgements}
This work was supported by China National Postdoctoral Program for Innovative Talents (BX20220392), China Postdoctoral Science Foundation (2022M711729), the National Science Fund for Distinguished Young Scholars (No.62025603), 
the National Natural Science Foundation of China (No. U21B2037, No. U22B2051, No. 62176222, No. 62176223, No. 62176226, No. 62072386, No. 62072387, No. 62072389, No. 62002305 and No. 62272401), 
Guangdong Basic and Applied Basic Research Foundation(No.2019B1515120049), and the Natural Science Foundation of Fujian Province of China (No.2021J01002,  No.2022J06001).

\bibliography{sn-bibliography}


\end{document}